\documentclass[12pt]{article}
\usepackage{setspace}
\newdimen\alglabelwd \setbox0\hbox{99.99}\alglabelwd=\wd0
\newcounter{algorithm}
\usepackage{secdot}
\usepackage{epsfig,amsmath,amsthm,amsfonts}
\usepackage{url}
\usepackage{natbib}

\newtheorem{theorem}{Theorem}

\newtheorem{proposition}[theorem]{Proposition}

\def\ni{\noindent}

\def\ascript{\mathcal{A}}
\def\bscript{\mathcal{B}}
\def\escript{\mathcal{E}}
\def\iscript{\mathcal{I}}
\def\gscript{\mathcal{G}}

\def\sm{\scriptsize}

\usepackage[normalem]{ulem}
\usepackage[color,final]{showkeys}
\usepackage{multirow}

\begin{document}

\begin{center}

{\LARGE  A General Stochastic Algorithmic Framework for Minimizing Expensive Black Box Objective Functions Based on Surrogate Models and Sensitivity Analysis}\\[12pt]


\footnotesize

\mbox{\large Yilun Wang}\\
School of Mathematical Sciences,
University of Electronic Science and Technology of China,
Chengdu, 611731,
China,
\mbox{yilun.wang@gmail.com}\\[6pt]

\mbox{\large Christine A. Shoemaker}\\
School of Civil and Environmental Engineering and School of Operations Research and Industrial Engineering,
Cornell University, Ithaca, New York 14853, USA,
\mbox{cas12@cornell.edu}\\[6pt]

\normalsize

\end{center}


\noindent We are focusing on bound constrained global optimization problems, whose objective functions are  computationally expensive black-box functions and have multiple local minima. 
The recently popular  Metric Stochastic Response Surface (MSRS) algorithm proposed by \cite{Regis2007SRBF}  based on adaptive or sequential learning based on response surfaces is revisited and further extended for better performance in case of higher dimensional problems. 
Specifically, we propose a new way to generate the candidate points which the next function evaluation point is picked from according to the metric criteria, based on a new definition of distance, and prove the global convergence of the corresponding. Correspondingly, a more adaptive implementation of MSRS, named ``SO-SA'', is presented. ``SO-SA'' is  
is more likely to perturb those most sensitive coordinates when generating the candidate points, instead of perturbing all coordinates simultaneously.
 Numerical experiments on both synthetic problems and real problems demonstrate the advantages of our new algorithm, compared with many state of the art alternatives.

\bigskip



\noindent\hrulefill

\section{Introduction} \label{sec-Introduction}
\subsection{Motivation and Problem Statement}
We are considering the optimization problem whose objective function is a computationally expensive deterministic function, i.e. each of its evaluation takes a substantial amount of computing time.
Moreover, the objective function $f(x)$ is often nonconvex and considered as a ``black box'' without analytic or numerically reliable  derivative information, i.e. very limited knowledge of $f(x)$ is available.  We take the parameter calibration of complex models as an example, where $f(x)$ is the distance between the measured data and the  output of  a parameterized complex computer simulation model, which are available for evaluation, but not for analytical inspection. This nonlinear regression procedure is to tune the model parameters in order to make the model match the measured data, and resulted model is used for prediction.  However, this kind of nonlinear regression problem brings huge computational challenges when the involved simulators are computationally expensive.  In particular, for relatively high dimensional problems considered in this paper, gradients obtained directly using finite differences are often prohibitively expensive because of too many required evaluations of $f(x)$, and can also be easily contaminated by numerical noise.

 Our main objective is therefore to design algorithms that can obtain a reasonably satisfying solution using only a small number of
 function evaluations because finding the ``exact" solution is often intractable. For example, in case of parameter calibration of a computationally expensive simulation model, only a small number of runnings of the simulation model are affordable to find the ``best" parameter values to make the model output match the measured data. 

 For this purpose, surrogates (or called response surface, metamodel) are often been used as an efficient adaptive learning to reflect the information of the underlying black box objective function $f(x)$ (\cite{Rios09comparision}). 
%
%
 While the use of a surrogate model is well suited for the type of optimization considered here,
most of this kind of algorithms are
known to become less efficient as the problem dimension increases, partially due to the difficulty to construct
a surrogate model $s(x)$ with satisfying global approximation property. 
Most of existing efforts to deal with high dimensional problems  are only trying to reduce a high dimensional problem into low dimensional problems either by decomposition or removing nonsensitive parameters, or other similar ideas, and have a lot of limitatations in practice due to their relatively restrict assumptions (\cite{Shan10survey}). In this paper, we will generalize an efficient  stochastic surrogate based algorithm framework  and
 present an efficient but simple implementation, for relatively higher dimensional problems. The
global convergence  is proved and its
remarkable practical performance is demonstrated by several widely
used synthetic testing problems and one real problems arising in
calibration of a groundwater model.

\subsection{Related Work}

For the computationally expensive function we are focusing on, its actual derivatives are rarely available,
although we hope that this circumstance will change as automatic differentiation technology
advances. But automatic differentiation does not always produce accurate derivatives and it is not applicable when the  complete source code for the objective function is not available.  On the other hand, choosing an appropriate step size for approximating
derivatives by finite differences is itself a difficult undertaking. The difficulties are
compounded by the expense of function evaluation and the fact that we could not afford to try a lot of
step sizes as many classical optimization algorithms do.   Moreover, finite-difference may
 be unreliable when the objective function is nonsmooth.  The above difficulties discourage the adaption of the
algorithms requiring derivative information. Therefore, in the following part, we will mainly review derivative-free global algorithms.

Let $f(x)$ be defined in $\mathfrak{D}$, which is a compact set of $\mathbb{R}^d$, and assumed to be a closed hypercube in this paper for simplicity.  
%
One kind of well-known methods for global minimization is to run a local optimization algorithm from multiple starting
points (i.e., use certain multistart procedure). Examples of multistart approaches include Multi Level
Single Linkage (MLSL) (\cite{Rinnooy87MLSL}), OQNLP (\cite{Ugray07OQNLP}) and etc. 
 The adopted popular  derivative-free local optimization methods are mostly  various direct search methods (\cite{Kolda03optimizationby}),
including Pattern Search by \cite{Torczon97Pattern}, Mesh Adaptive Direct Search (MADS) (\cite{Audet06MADS}), 
and Implicit filtering by \cite{Gilmore95IF}.
Notice that although many traditional direct search methods (\cite{Dennis91directsearch}) are local search approaches, they are able to escape form local optima to find better to certain degree if not global optimal solutions (\cite{Panos98tabu-pattern}).
Surrogates models can be applied to some of the above to  improve their performance in cases of  computationally expensive objective function. For example, NOMADm, a MATLAB implementation of the MADS has incorporate several response surfaces including Kriging, radial basis functions, Nadaraya-Watson estimator, and support vector machine (\cite{cortes1995support}).
The trust-region methods and the pattern search methods using the
approximation models were well studied by \cite{Dennis95KrigingPatternSearch}.
Derivative-free trust-region
 methods for unconstrained optimization (\cite{Conn97DFO, Powell02UOBYQA, Powell06NEWUOA}) typically rely
 on local quadratic models of the objective function that are built by interpolation using a subset
 of previously evaluated points.

As for the global search algorithms, one important kind of  methods are deterministic global search algorithms, including Lipschitizian-based partition algorithms, for example, the DIRECT (DIvide a hyper-RECTangle) algorithm (\cite{Jones01Direct, Jones93DIRECT}), and the Branch-and-Bound algorithms, and  Multilevel  Coordinate Search by \cite{Huyer99MCS}.
However, for high dimensional problems, those methods based on partitioning the solution space have a worst case exponential complexity, and therefore are not suitable for computationally expensive functions, though they might be rigorous  and are able to give you a provably global solution.
Another approach for finding the global minimum of an objective function with multiple local
minima is to use heuristic methods or stochastic methods including simulated annealing, evolutionary algorithms (e.g., genetic algorithms,
evolution strategies and evolutionary programming), ant algorithms, scatter search by \cite{Glover98SS, Laguna03SS}, Covariance Matrix Adaption Evolution (CMA-ES) by \cite{Hansen03CMAES}.
However, these algorithms are usually designed for functions that are cheap to
evaluate. When $f(x)$ is expensive,  the algorithms are expected to  maximize the information gained, since a huge number of evaluations of $f(x)$
are often not affordable.  In such cases, response surface can be applied in order to further improve their performance. For example, SSKm (Scatter Search with Kriging for Matlab) presented by \cite{EGEA09SSKm} is able to improve the performance of scatter  search on computationally expensive problems by liking a scatter search method with a Kriging interpolation. Radial basis function interpolation has also been
 used to accelerate an evolution strategy by \cite{Regis04IEEE}.  Meanwhile, among response-surface--free algorithms, there also exist some specifically design for relatively high dimensional computationally expensive functions, such as Dynamically Dimensioned Search (DDS) by \cite{Tolson07DDS}, which outperforms many other alternatives especially for cases where the objective function $f(x)$ is not smooth, but rough.

Considering that  the surrogate can help reduce the expensive function evaluations of $f(x)$ to find its global optimum, ones have been developing many new surrogate based global optimization algorithms in the past years. This kind of algorithms are usually iterative procedures to keep updating the resulted surrogate surface and the key point is how to pick the next function evaluation point in order to balance the two somehow conflicting purposes, improving the accuracy of the response surface and finding the possible minimum of $f(x)$ based on the current response surface.
 \cite{Jones98EGO} developed a kriging-based global optimization method called
 EGO where the next iterate is obtained by maximizing an expected improvement function. \cite{Aleman09EGO} used a variant of the EGO method by \cite{Jones98EGO} to optimize beam
 orientation in intensity modulated radiation therapy (IMRT) treatment planning. \cite{Villemonteix09IAG} also developed a kriging-based method called IAGO that uses minimizer entropy as a
 criterion for determining new evaluation points.  \cite{Gutmann01RBF} used RBFs to develop a global optimization method where the next iterate is obtained by
 minimizing a bumpiness function and variants of this method have been developed by \cite{Bjork2001},  and \cite{Holmstrom08ARBF}.  





Recently, a kind of efficient Metric Stochastic Response Surface (MSRS, for short) algorithm was proposed by \cite{Regis2007SRBF}.  The criterion of selecting the next function evaluation is minimizing a  merit function, which is a the weighted sum of the response surface value and the distance with the set of previous function evaluation points.  A cycle of weights is adopted to either emphasize more on the response surface value or on the distance criteria.  Instead of directly minimizing the merit function, they  generate a series of random candidate points and then
pick the next function evaluation point with the minimal merit function value. It is a computationally efficient way to choose the next function evaluation point and works well in practice.  The candidate points might be generated in different ways and  the corresponding algorithms might have significantly different performances. For example, in their paper,  the candidate points are generated by perturbing the current best solution via either uniform distribution ( Global Metric Stochastic Response Surface Algorithm) or normal distribution (Local Metric Stochastic Response Surface).  Both algorithms can achieve global convergence almost surely due to the random generation of the candidate points, and the latter one behaviors better than the former one in most cases, especially when the number of function evaluations of $f(x)$ is small.  





\subsection{Our Contributions and Paper Organization}
 In this paper, we are focusing on relatively high dimensional multimodal problems ($d\ge 30$), and the allowed maximal number of function evaluations is typically very small ($500$, for example), considering the dimension of the problem. 
Our proposed algorithm is an extension and generalization of the Metric Stochastic Response Surface
(MSRS, for short) algorithm by \cite{Regis2007SRBF}, which has proved to
be more efficient for relatively low dimensional problems than many other existing algorithms. 
In this paper, we further analyze and extend of  the key features of  MSRS to make it more suitable for relatively high dimensional problems. 

First, we extended the MSRS algorithm, in the aspect of generating the random candidate points from which the next function evaluation point is chosen, and still guarantee the global convergence in the probabilistic sense.  
In particular, while each candidate point is generated by perturbing the current best solution in every coordinate in the original MSRS,   we generalize it by allowing the probability for each coordinate to be perturbed  can be smaller than $1$.   
 In addition,  the  probability value for each coordinate   is not necessarily the same and we propose to make it  dependent on the current learning of the underlying objective function $f(x)$.  For example, the
 local sensitivity information is used to set the probability values of different coordinates  in this paper. 
 In particular, the sensitivity analysis is performed on the surrogate surface, because we could not afford to perform the sensitivity analysis on $f(x)$,  due to the very limited number of allowed function evaluations of $f(x)$ and the high dimensionality of it. 
 The more
  sensitive coordinates have  higher probabilities  to be chosen to be perturbed when generating  candidate points. 

Secondly, we further analyzed and  modified the  criteria for selecting the next function evaluation point from random candidate points, introduced by \cite{Regis2007SRBF}. We would  like to explain that the combination of this  criteria with the candidate points method is quite an efficient and reasonable way for relative high dimensional problems, which can be considered as ``constrained random search". 

Thirdly, we perform a comparison of our algorithm with many widely used state of the art algorithms including EGO, SSKm, ESGRBF, LMSRBF, Nomads-DACE, DDS, and DYCORS,  for the relatively high dimensional problems. The above algorithms are all designed for computationally expensive functions. To our best knowledge, this kind comparison of the above algorithms has not been done in literatures before. 

In Section \ref{Sec:Alg}, we  first review MSRS by \cite{Regis2007SRBF} and propose our generalizations and modifications, for relatively high dimensional problems.  In Section \ref{Sec:proof}, we proved that global convergence is guaranteed after the generalizations. In Section \ref{Sec:SA_SO},  we propose our new implementation of MSRS named ``SO-SA'', which makes use of the local sensitivity information when generating the candidate points.  In Section \ref{Sec:Exp}, we compared the new algorithm ``SO-SA'' with several state of the art alternative algorithms for several typical test problems and demonstrate its advantages over them. In Section \ref{Sec:future} , summary and future work will be presented.

\section{Revisiting and Extensions of MSRS} \label{Sec:Alg}

\subsection{Review of  Algorithmic Framework of MSRS} \label{sub:LMSRBF}


In the paper by \cite{Regis2007SRBF},  a metric stochastic response surface (MSRS, in short) method is introduced, where the next function evaluation is chosen from a sequence of random candidate points according to certain criteria, which is a combination of the response surface value and the distance to the previously evaluated points of $f(x)$.  
In this paper, we are going to generalize it and 
%
before moving forward, we first review the algorithmic framework of the MSRS and briefly explain its steps on by one.

\ni\hrulefill

\ni {\bf The Algorithmic Framework of MSRS:}\\

\ni {\bf Inputs:}
\begin{itemize}
\item[(1)] A continuous real-valued function $f$ defined on a
compact hypercube ${\mathfrak{D}}=[\vec{a},\vec{b}] \subseteq \mathbb{R}^d$.

\item[(2)] A particular response surface model, e.g. radial basis
functions or neural networks.

\item[(3)] A set of initial evaluation points ${\iscript} =
\{x_1,\ldots,x_{n_0}\} \subseteq \mathfrak{D}$, e.g. a space-filling
experimental design.


\item[(4)] The maximum number of function evaluations allowed
denoted by $N_{\max}$.
\end{itemize}

\ni {\bf Output:} The best solution encountered by the algorithm.
\vspace{2mm}

\ni {\bf Step 1} {\bf (Do Costly Function Evaluation)} Evaluate
the function $f$ at each point in ${\iscript}$. Set $n = n_0$ and
set ${\ascript}_n = {\iscript}$. Let $x_n^*$ be the point in
${\ascript}_n$ with the best function value and $f_n^*=f(x_n^*)$. 

\vspace{2mm}



\ni {\bf Step  2} While $(n < N_{\max})$

\begin{itemize}

\vspace{-2mm}

\item[] {\bf Step 2.1} {\bf (Fit/Update Response Surface Model)}
Fit/update the response surface model $s_n(x)$ using the data
points ${\bscript}_n=\{(x_i,f(x_i)):i=1,\ldots,n\}$.

\item[] {\bf Step 2.2} {\bf (Randomly Generate Candidate Points)}
Randomly generate $t$ points $\Omega_n=\{y_{n,1},\ldots,y_{n,t}\}$
in $\mathbb{R}^d$.
For each $j=1,\ldots,t$, if $y_{n,j} \not\in \mathfrak{D}$, then
replace $y_{n,j}$ by the nearest point in $\mathfrak{D}$. We refer to
the points in $\Omega_n$ as {\it candidate points}.

\item[] {\bf Step 2.3} {\bf (Select the Next Function Evaluation
Point)} The merit function is  considering both response surface value $s_n(x)$ and the distance from the previous evaluated points ${\bscript}_n=\{(x_i,f(x_i)):i=1,\ldots,n\}$. 
The candidate points corresponding to the smallest merit function value, among the $t$
candidate points in $\Omega_n$, is selected as the next
evaluation point $x_{n+1}$. 



\item[] {\bf Step 2.4} {\bf (Do Costly Function Evaluation)}
Evaluate the function $f$ at $x_{n+1}$.



\item[] {\bf Step 2.5} {\bf (Update Best Function Value)} Update the best function value encountered so far, i.e. if $f(x_{n+1})<f_{n}^*$, then
$x_{n+1}^{*}=x_{n+1}$; otherwise $x_{n+1}^*=x_{n}^*$.   $f_{n+1}^*=f(x_{n+1}^*).$ 

\item[] {\bf Step 2.6} {\bf (Update Information)}
${\ascript}_{n+1} := {\ascript}_n \cup \{x_{n+1}\}$;
${\bscript}_{n+1} := {\bscript}_n \cup \{ (x_{n+1},f(x_{n+1}))
\}$.  
Reset $n := n+1$.

\end{itemize}

\ni {\bf Step 3} {\bf (Return the Best Solution Found)} Return
$x_{\mbox{\sm $N_{\max}$}}^*$.


\ni\hrulefill

Like many other surrogate optimization algorithms, MSRS starts by evaluating the expensive objective function $f(x)$ at an initial set of points usually generated by a space-filling experimental design in Step $1$. In this paper, the specific space-filling experimental design adopted is the version by \cite{Ye00OSLHD}, though other experimental design methods might be also applicable here. 


Step $2$ is the main body of MSRS  and it is an iterative procedure until the computational budget, i.e. the number of the maximal allowed function evaluations of $f(x)$ is reached. 

Step $2.1$ might use any type of response surfaces such as Radial Basis Function (RBF) interpolation  (\cite{Powell1992RBF,Buhmann03RBF}, kriging \cite{Cressie91Book, Sacks1989DACE}), or neural networks. In all the above cases, the complicated objective function $f(x)$ is expressed as a weighted sum of many simple functions. The predictions of $f(x)$ can then be made based on the adopted surrogate surface.
In this paper, we adopts the RBF interpolation  and a complete introduction to the RBF interpolation was given in \cite{Buhmann03RBF}.  

Step $2.2$ and Step $2.3$ together are about how to determine the next function evaluation point, or called the design point in some literature. This scheme is the key component for any response surface based optimization algorithm.  Different surrogate optimizations adopt different criteria to determine the next function evaluation point, usually in the form of minimizing a so-called merit function. For example, a native way is to let the merit function be the current surrogate and use its minimum as the next function evaluation point. However, it is highly likely that the chosen next function evaluation point is very close to the current best solution which has already been evaluated,  if not ``exactly" the same, and therefore, the best solution and the surrogate are hardly improved. In such cases, the algorithm quickly converges to a local (global) minimum of the surrogate, which is often not even local minimum of the true function $f(x)$.  Therefore the approximation error of the surrogate also needs to be considered when designing the merit function and therefore most popular existing surrogate optimization algorithms try to  make the merit function better balance the  improvement of the surrogate quality and the exploitation of the current surrogate. For example, EGO by \cite{Jones98EGO}, which is based on Kriging surrogate, makes use of the explicit approximation error formula of the surrogate to estimate where is least well approximated. The next function evaluation point is  selected where the kriging predictor value is small (local search) and the kriging mean squared error is high (global search).  Putting some emphasis on searching where the error is high ensures that we improve the approximation accuracy of the surrogate surfaces as well as encourage the global search.
%
%

However, there is no explicit error estimation formula for many other the surrogate response surfaces such as the radial basis function interpolations. Thus  MSRS  proposed to adopt the distance with set of the previously evaluated points to estimate the surrogate approximation accuracy of a given point; roughly, the bigger  the distance, the larger  the approximation error is.  Correspondingly, their merit function was a weighted combination of surrogate values and distance to the previous evaluated points and performed well in all of their test problems.  

Instead of directly minimizing the merit function, \cite{Regis2007SRBF} proposed to generate a sequence of random candidate points and pick the one with lowest merit function value as the next function evaluation point, originally due to its  computational efficiency. In addition, the great practical performance of the candidate point method has also been demonstrated in  \cite{Regis2007SRBF}, for relatively low dimensional problems $(d \le 10).$  In this paper, we will further analyze the method of candidate points and generalized its generation method, and present more advantages of the candidate points methods, especially for relatively high dimensional problems.

In Step $2.4$, the objective function is evaluated at the selected point.

In Step $2.5$, we update the current best solution, if the function value $f(x)$ at the  the newly selected point is smaller than the function values $f(x)$ at all the previous function evaluation points.

In Step $2.6$, we update the sets related with the available functions valuations of $f(x)$, by adding the newly selected function evaluation point.

The efficiency of MSRS mainly depends on two aspects, one is the definition of the merit function that will be further analyzed in Section \ref{Sub:Step23}, and the other is quality of the generated candidate points that will be analyzed in section \ref{Step22}.  

\subsection{Revisiting and Modification of the Merit Function} \label{Sub:Step23}



The merit function proposed in MSRS by \cite{Regis2007SRBF} consists of the estimated function value from the current response surface $s_n(x)$ and the minimum distance from previously evaluated points. 
 Specifically,  the corresponding merit function  is a weighted sum of the response surface value and the distance as follows:    \begin{equation}\label{eq:utiDef}
    u(x)=w_n^S \frac{s(x)-s_{\min}}{s_{\max}-s_{\min}} + w_n^D \frac{d_{\min}-d(x)}{d_{\max}-d_{\min}}
    \end{equation}
    where $s(x)$ is the current surrogate model, $d(x)=\min(dist(x, y_i))$ with $y_i$ being the previous evaluated point $(i=1, \ldots, n)$ and $d_{\min}=\min(d(x)), x \in \mathfrak{D}$; $d_{\max}=\max(d(x)), x \in \mathfrak{D}; s_{\min}=\min(s(x)), x \in \mathfrak{D};$ and $s_{\max}=\max(s(x)), x \in \mathfrak{D}$,.
The set of
nonnegative weights $\{(w_n^S, w_n^D): n=n_0, n_0+1,\ldots \}$ where $w_n^S +w_n^D=1$ is aiming to control the balance of the surrogate surface value criteria and the distance criteria. Once $w_n^S$ is determined, $w_n^D$ is determined correspondingly since $w_n^D =1-w_n^S.$
   The transition between local search and global search depends on the value of the weight $w_n^{S} $ (correspondingly $w_n^{D}$).
A large $w_n^{S} $ encourages  the local refining while a smaller $w_n^{S}$ more encourages the global exploration.


We would like to revisit  the merit function from  point of view of adaptive machine learning. The two parts of the merit function correspond to   ``most informative" data point and  selecting ``most uncertain" data point, respectively.  These two goals are in fact not necessarily always conflicting, i.e. the purpose of selecting the ``most uncertain" data point is often to avoid the local trapping of the current response surface and help to find the ``most informative" data point (the one with lower objective function value) in the following iterations.

 About the setting of $w_n^S$ and $w_n^D$,    the cycled weights consisting of large values and small values  is  proposed by \cite{Regis2007SRBF} and this kind of  cycling of the weights is only a deterministic case. It fact, the cycling of weights means to say that we do not know the size of the approximation error of the current surrogate surface. 
 Therefore, without any other prior information, in this paper we propose to  just turn to the randomness, i.e. randomly pick a value between $[0, 1]$. It is a simpler way and better reflect the spirit of the definition of the merit function. Prescribed cycled values such as $[03, 0.5, 0.7, 0.95]$ in \cite{Regis2007SRBF} are also only kind of arbitrarily set and has no solid theoretical support. Furthermore,
 %
we will also adopt a greedy way in this paper. When we find a significantly better new solution with certain $w_n^{S}$, we would like to keep use it
until we fail to find a significantly better solution. This strategy often works well in the case of high dimensional problems and very limited number of function evaluations.   

\ni \hrulefill

{\bf Step 2.3} {\bf (Select the Next Function Evaluation
Point)} Use the information from the response surface model
$s_n(x)$ and the data points
${\bscript}_n=\{(x_i,f(x_i)):i=1,\ldots,n\}$ to select the
evaluation point $x_{n+1}$  from the $t$ random
candidate points in $\Omega_n$.

\begin{itemize}
\item[] {\bf Step 2.3.1} {\bf (Estimate the Function Value of Candidate Points)} For each $x \in \Omega_n,$ compute $s_n(x)$. Also, compute
$s^{\max}_n =\max\{s_n(x): x\in \Omega_n\}$ and $s^{\min}_n = \min\{s_n(x): x\in \Omega_n\}.$
\item[] {\bf Step 2.3.2} {\bf (Compute the Score Between $0$ and $1$)} For each $x \in \Omega_n$, compute $V_n^S(x)=(s_n(x)-s_n^{\min})/(s_n^{\max}-s_n^{\min})$ if $s_n^{\max} \neq s_n^{\min}$ and $V_n^S=1$ otherwise.
\item[] {\bf Step 2.3.3} {\bf (Determine the Minimum Distance from Previously Evaluated Points)} For each $x\in \Omega_n,$ compute $d_n(x)=\min_{1\le i\le n} \|x-x_i\|_2$. Also, compute $d_n^{\max}=\max\{d_n(x): x\in \Omega_n\}$ and $d_n^{\min}=\min\{d_n(x): x\in \Omega_n\}.$
\item[] {\bf Step 2.3.4} {\bf (Compute the Score between 0 and 1 for the Distance Criterion)} For each $x\in \Omega_n,$ compute $V_n^D=(d_n^{\max}-d_n(x))/(d_n^{\max}-d_n^{\min})$ if $d_n^{\max} \neq d_n^{\min}$ and $V_n^{D}(x)=1$ otherwise.
\item[] {\bf Step 2.3.5} {\bf (Determine the Weights $w_n^S$ and $w_n^D$)} Randomly pick $w_n^D$ in $[0,1]$ and $w_n^S=1-w_n^D.$
\item[] {\bf Step 2.3.6} {\bf (Compute the Weighted Score)} For each $x \in \Omega_n,$ compute $u_n=w_n^S V_n^S(x)+ w_n^DV_n^D(x)$
\item[] {\bf Step 2.3.7} {\bf (Select Next Evaluation Point)} Let $x_{n+1}$ be the point in $\Omega_n$ that minimizes $u_n.$
\end{itemize}
\ni \hrulefill
\subsection{Further Analysis and Extension of Random Candidate Points Method}\label{Step22}
 Here we would like to further analyze why the next function evaluation point is preferred to be   picked from the random candidate points rather than directly minimizing of the merit function $u(x)$ defined as \eqref{eq:utiDef}, besides  
the original motivation in \cite{Regis2007SRBF} that   direct minimization of  the merit function  is more computationally costly, because $s(x)$ is non-convex and
 the  computation of $s_{\max}$ and $s_{\min}$  is not straightforward. 

First of all,  random search methods, as we know, have been shown to have a potential to solve
high dimensional problems efficiently in a way that is often not possible for deterministic
algorithms in \cite{Zabinsky03SAS}. Specifically,  if one is willing to accept a weaker
claim of being correct with an estimate that has a high probability of
being correct, then a stochastic algorithm might provide such an estimate
using much less function evaluations than deterministic algorithms. It is very abstractive for our cases. 
  %
 Besides easing the curse of dimensionality,  the randomness brought by the  candidate points also helps 
avoid the local traps of the nonconvex $f(x)$ and $s(x)$. However, for relatively high dimensional problems, the existing random search strategies might suffer from slow convergence rate, which is becoming even more unacceptable when the objective function is a computationally expensive function. Therefore, the merit function $u(x)$ can be considered as a guidance for the random search in order for better practical convergence rate and we name our scheme of selecting the next function evaluation point as  ``guided random search".  Finally, we would like to mention that the probability distribution of generating the perturbations can be beyond the normal distribution and uniform distribution adopted by \cite{Regis2007SRBF}.

Secondly, when determine the next function evaluation point, we prefer a more greedy strategy, i.e.  searching around the current best solution to obtain an even better solution, due to the very limited total allowed number of function evaluations of $f(x)$ and relatively high dimensionality of the searching space, with the aim to  
find out a reasonably good enough solution quickly. In such cases,  the method of  random candidate points has its advantages over the directly minimizing the merit function $u(x)$,  because it can well preserve the already achieved searching progress due to the fact that the candidate points are mostly generated around the current best solution, though the stochastic property still preserves the  global exploration, when the normal distribution is adopted to generating the perturbations. 
In this paper, we proposed  to use a new definition of ``neighborhood" in cases of high dimensional problems.  In the original implementation of MSRS by \cite{Regis2007SRBF}, the neighborhood of the current best solution is measured by the perturbation magnitudes.  Here 
the neighborhood of the current best solution is measured by the number of the coordinates perturbed from the current best solution. 
That is to say, a point is closer to the current best solution than others if it is generated by perturbing a smaller number of coordinates from the current best solution. In particular,
%
a candidate point can be produced via perturbing the current best solution only in a subset of all the coordinates, instead of  all of them, as the original implementation of MSRS   did. 
 We set the probability of a coordinate being perturbed as a positive number in [$C_1$,1], where $C_1$ is a very small positive constant close to $0$. The original implementations of MSRS can be considered as a special case of our new framework, where the probability of being selected is always $1$.
Notice that each candidate point may perturb different coordinates and tons of candidate points are generated during each iteration.
Therefore  a relatively
diverse set of search directions in each iteration of the algorithm is still obtained and correspondingly the searching space is almost fully explored, though preferably around the current best solution.  

Thirdly, the method of random candidate points bring more flexibility than directly minimizing the merit function $u(x)$. We can adaptively generate the candidate points by take the specific learned property of the underlying function $f(x)$ into consideration.  We will show how we take the sensitivity information into the consideration for setting the probability value of a coordinate to be selected when generating candidate points, in Section \ref{Sec:SA_SO}. 

\section{Proof of Global Convergence} \label{Sec:proof}

While we generalize the way of generating the candidate points in the above section, the global convergence still be preserved and we will prove it in this section.   As we have seen,  the generalization is mainly in two aspects, where one is that the probability of a coordinate to be selected to  be perturbed when generating a candidate point might be a positive value less than $1$, not necessarily equal to $1$, and the other is the probability distribution of generating the perturbations can be beyond the normal distribution and uniform distribution adopted by \cite{Regis2007SRBF}.

In the original paper of MSRS by \cite{Regis2007SRBF}, the authors have already proved that if the random candidate points of each iteration satisfy the following two conditions, the algorithm  is guaranteed
to converge to the global minimum almost surely provided that the algorithm is allowed to run indefinitely. In the following parts, we will show that the random candidate points produced by our generalized way still satisfy these two conditions. 
\begin{itemize}
\item Condition [1]: For each $n\ge n_0$, $Y_{n,1}$, $Y_{n,2}$, $\ldots,$ $Y_{n,t}$ are conditionally independent given the random vectors in $\escript_{n-1}$
\item Condition [2]: For any $j=1, \ldots, t, x\in \mathfrak{D}$ and $\delta>0$,  there exists $\nu_j(x,\delta)>0$ such that
$$P[Y_{n,j} \in B_k(x,\delta)\cap \mathfrak{D}|\sigma(\escript_{n-1})] \ge \nu_j(x,\delta)
$$
for all $n\ge n_0$. Here $B(x,\delta)$ is the open ball of radius $\delta$ centered at x and $\sigma(\escript_{n-1})$ is the $\sigma$-field generated by the random vectors in $\escript_{n-1}$.
\end{itemize}


When generating the candidate points, we have followed the Condition [1].  In the following parts, we will show that the generated candidate points of our generalized way still satisfy Condition [2]. 
Specifically,  let the probability of the $i$-th coordinate  being chosen to be perturbed in $n$-th
iteration is denoted as $P_{n,i}$. While $P_{n,i}$ is always $1$ in original MSRS, it can be any positive real number in $(C_1,1]$, where $C_1$ is a very small positive constant which can be very close to $0$.   Let $F_{n,j,i}\ge 0$ be the continuous density function for  the coordinate $i$ of the $j$-th candidate point of the $n$-th iteration when generating the random perturbation. 
In this paper, we propose that if the density $F_{n,j,i}\ge 0$ function satisfies certain mild conditions, then Condition [2] always holds.  





\begin{proposition}\label{prop:C2}
Condition [C2] holds,  if there exist  constants  $C_1>0$ and $C_2>0$, such that the probability $P_{n,i}>C_1$ and the continuous density function $F_{n,j,i}>C_2$, for every $n \in \mathcal{G}^2$ and $1\le i\le d$, $j=1,\ldots, t$. 
\end{proposition}

\begin{proof}
Define $\psi_{\mathfrak{D}}(\delta):=\inf_{x\in \mathfrak{D}} \mu(B(x,\delta)\cap \mathfrak{D})$, where $\mu$ is the Lebesgue measure on $\mathbb{R}^d$. Observed that for the compact hypercube $\mathfrak{D}$, we have $\psi_{\mathfrak{D}}(\delta)>0$ for any $\delta >0$.

Fix $1\le j\le t$, $x \in \mathfrak{D}$ and $\delta>0$.
Since we are considering a compact set $\mathfrak{D}$
 it follows that $Y_{n,j}$ has a conditional
density given $\sigma(\escript_{n-1})$ for each $n \ge n_0$ and
this is given by
\begin{eqnarray*}
g_{n,j}(y\ |\ \sigma(\escript_{n-1})) \ge g_{n,j}(n \in \gscript^2,y\ |\ \sigma(\escript_{n-1})) \ge C_1 C_2^d >0\\
\end{eqnarray*}
for all $y \in {\mathfrak{D}}$. Hence,
\begin{eqnarray*}
P[Y_{n,j} \in B(x,\delta)\cap \mathfrak{D}\ |\ \sigma(\escript_{n-1})] 
&\ge&\int_{B(x,\delta)\cap \mathfrak{D}} \hspace{-10mm} g_{n,j}(y\ |\
\sigma(\escript_{n-1}))\ dy \\&\ge&C_1 C_2^d\mu(B(x,\delta)\cap \mathfrak{D}) \\&\ge& C_1
C_2^d\psi_{\mathfrak{D}}(\delta)>0,
\end{eqnarray*}
for any $n \ge n_0$. 

So, Condition~[C2] also holds.
\end{proof}

Therefore, according to the results in the original MSRS paper by \cite{Regis2007SRBF}, if  Condition ~[1] and Condition ~[2] hold, the global convergence of MSRS will be obtained  in a probabilistic
sense, and the theorem is restated as follows. 

\begin{theorem} \label{Thm:convergence}
Let $f$ be a function defined on ${\mathfrak{D}} \subseteq
\mathbb{R}^d$ and suppose that $x^*$ is the unique global
minimizer of $f$ on $\mathfrak{D}$ in the sense that $f(x^*) = \inf_{x
\in \mathfrak{D}}f(x)>-\infty$ and $\inf_{\substack{x \in {\mathfrak{D}} \\
\|x-x^*\| \ge \eta}} f(x) > f(x^*)$ for all $\eta > 0$. Suppose
further that the MSRS method generates the random vectors
$\{X_n\}_{n \ge 1}$ and $\{Y_{n,1},\ldots,Y_{n,t}\}_{n \ge n_0}$.
Define the
sequence of random vectors $\{X_n^*\}_{n \ge 1}$ as follows:
$X_1^*=X_1$ and $X_n^*=X_n$ if $f(X_n) < f(X_{n-1}^*)$ while
$X_n^* = X_{n-1}^*$ otherwise. Then $X_n^* \longrightarrow x^*$
almost surely.
\end{theorem}
Proposition \ref{prop:C2} indicates that the perturbation of a given coordinate can be generated via a large family of probability distributions, not limited to normal distribution and uniform distribution. In the paper by \cite{Regis2007SRBF}, the perturbations generated by the normal distribution or uniform distributions were proved to satisfy Condition [2] when all the coordinate are perturbed simultaneously.

While the global convergence is of theoretical significance, its proof here in fact depends on the probability of all the coordinates being simultaneously selected is still a positive value, though probably very small. That is to say, the probability value of each coordinate to be selected is also a balance between encouraging global exploration and local searching. In our cases, where the problem dimension is high and the available function evaluation functions are very limited, we might prefer small probability values, which squint more toward local searching. 


\section{Using Sensitivity Information when Generating Candidate Points} \label{Sec:SA_SO}
As for the specific implementation of Step $2.2$ in SO-SA, we adopts the new definition of ``neighborhood", which is measured by the number of the coordinates perturbed from the current best solution. The similar idea has been applied in  the DDS (Dynamically Dimensioned Search) algorithm by \cite{Tolson07DDS} and the DYCORS algorithm \cite{Regis11DCS}. 
%
Specifically, the random candidate points are obtained by perturbing only a subset of the coordinates of the current
best solution. Moreover, the probability values of perturbing different coordinates are the same for each iteration, and  decrease as the algorithm reaches the
computational budget.

In this paper, we further propose to 
set
 different probability values for different coordinates according to the local sensitivity information of the current best solution estimated on the current response surface. The motivation is that along the more sensitive coordinates, we would like to put more dense candidate points in order for the refinements while putting less on less sensitive ones, i.e. setting smaller number of the perturbed coordinates. 
In particular, we rank the coordinates  from the most sensitive ones to least sensitive ones, and  the more sensitive ones are assigned a higher probability value to be chosen and the less sensitive ones are given a much smaller value when generating a candidate point. 
%
%
%
The detailed description of our adopted particular sensitivity analysis methods is presented in Section \ref{subsec:SA} and we call this specific implementation of MSRS as SO-SA. 

The detailed description of Step $2.2$ in SO-SA is presented as follows.


\ni\hrulefill

\ni { {\bf Step 2.2 } {\bf (Generate Random Candidate Points )}}\\
\begin{itemize}

\item[] {\bf Step 2.2.1} {\bf(Perform a Local Sensitivity Analysis)} A prescribed sensitivity analysis method is performed on the response surface $s(x)$ around the current best solution $x^*_{n}$.
\item[] {\bf Step 2.2.2} {\bf(Determine the Probability  for Each Coordinate to be Selected as the Perturbed Coordinate)}
According to local sensitivity analysis of the current best solution, the more sensitive coordinate has a higher probability value and the less sensitive coordinate has a lower probability. The probability of the $i-$th coordinate for denoted as $p_{n,i} \in [C_1,1]$, where $C_1$ is a prescribed positive number, which might be very small. 
\item[] {\bf Step 2.2.3} {\bf(Generate Random Candidates Points based on Local Sensitivity Analysis)}  Randomly generate $t$ points $\Omega_n=\{y_{n,1},\ldots,y_{n,t}\}$
in $\mathbb{R}^d$.
For each $j=1,\ldots,t$, if $y_{n,j} \not\in \mathfrak{D}$, then
replace $y_{n,j}$ by the nearest point in $\mathfrak{D}$. We refer to
the points in $\Omega_n$ as {\it candidate points}. For each candidate point $y_{n,j}$, its generating procedure is listed
as follows.

\begin{itemize}
\item[] {\bf Step  2.2.3(1)} {\bf(Select Coordinates to be Perturbed)}
For each coordinate $i$ $(i=1,2,\ldots, d)$, generate a random value $\omega_{n,j,i}$ via uniform distribution in [0,1];
 Denote   coordinates to be perturbed as  the set $I^{n,j}_{perturb}\doteq\{i: \omega_{n,j,i} \le p_{n,i}\}$. 

\item[] {\bf Step 2.2.3(2)} {\bf(Determine the Standard Deviations for Generation of the Perturbations)}. For each selected coordinate, the perturbation
is generated via normal distribution with mean being 0, and
 the standard deviation being randomly selected from a series of candidate values ranging from large to small. 
Therefore, different perturbed coordinates might have different standard deviations.

\item[] {\bf Step  2.2.3(3)} {\bf(Generate a Candidate Point)} Generate the $j$th candidate point $y_{n,j}$ 
by $y_{n,j}=x_{n}^* +\delta^n_{j}$ where $\delta^n_{j,i}=0$ for all $i\not\in I^{n,j}_{perturb}$ and $\delta^n_{j,i}$ is a normal random variable with mean $0$ and standard deviation settled in Step 2.2.3(2),  for all $i \in I^{n,j}_{perturb}$. 

\item[] {\bf Step  2.2.3(4)} {\bf(Ensure Candidate Point is in Domain)} If $y_{n,j} \not \in \mathfrak{D}$, replace $y_{n,j}$ by a point in $\mathfrak{D}$ obtained by performing successive reflection of $y_{n,j}$ about the closest point on the boundary of the hypercube $\mathfrak{D}$.

\end{itemize}
\end{itemize}

\ni\hrulefill

\subsection{Local Sensitivity Analysis on the Response Surfaces} \label{subsec:SA}
While there have existed a variety of  local sensitivity analysis methods,  ones can choose appropriate methods according to their own preference. We just use the sensitivity indices based on the existing univariate perturbations and bivariate perturbations as showed in this paper. 

Given a function $y(x)$ and a point $\bar{x}\in \mathbb{R}^d$, the first kind of sensitivity index is merely the univariate perturbations, or central finite difference,  as follows:
\begin{equation}\label{Def:SIkR}
SI_{i}^{1,\Delta}=\left|
y(\bar{x}^{(i^{+},\Delta)})-y(\bar{x}^{(i^{-},\Delta)})\right|
\end{equation}
where
%
%
\begin{equation}
\bar{x}^{(i^{+},\Delta)}=[
\bar{x}_1, \ldots, \bar{x}_{i-1}, \bar{x}_i+\Delta, \bar{x}_{i+1}, \ldots, \bar{x}_d ]
\end{equation}
and
\begin{equation}
\bar{x}^{(i^{-},\Delta)}=[
\bar{x}_1, \ldots, \bar{x}_{i-1}, \bar{x}_i-\Delta, \bar{x}_{i+1}, \ldots, \bar{x}_d ]
\end{equation}

The second and the third kinds of sensitivity indices are both based on the univariate perturbations together with the
bivariate perturbations. We consider the following matrix $L^\Delta$, which consists of elements $L^\Delta_{i,j}$ that is the
largest among all possible perturbations along coordinates $i$ and $j$ in terms of the absolute value. 
The definition of $L^\Delta$ is as follows:

\begin{eqnarray*} L^{\Delta}=
\left(
  \begin{array}{cccc}
     \ell^{\Delta}_{1,1} &\ell^{\Delta}_{1,2}& \ldots & \ell^{\Delta}_{1,d} \\
    \ell^{\Delta}_{2,1} &\ell^{\Delta}_{2,2} & \ldots & \ell^{\Delta}_{2,d} \\
    \vdots & \vdots & \ddots & \vdots \\
    \ell^{\Delta}_{d,1} & \ell^{\Delta}_{d,2} & \ldots & \ell^{\Delta}_{d,d} \\
  \end{array}
\right)
\end{eqnarray*}
where
\begin{eqnarray}\nonumber
\ell^{\Delta}_{i,j}&=&\max(|y(\bar{x}^{(i^{+},j^{+},\Delta)})-y(\bar{x})|,  |y(\bar{x}^{(i^{-},j^{+},\Delta)})-y(\bar{x})|, |y(\bar{x}^{(i^{+},j^{-},\Delta)})-y(\bar{x})|, |y(\bar{x}^{(i^{-},j^{-},\Delta)})-y(\bar{x})|)\\\nonumber
\ell^{\Delta}_{i,i}&=&\max(|y(\bar{x}^{(i^{+},\Delta)})-y(\bar{x})|, |y(\bar{x}^{(i^{-},\Delta)})-y(\bar{x})|)
\end{eqnarray}
\begin{eqnarray}\nonumber
\bar{x}^{(i^{+},j^{+},\Delta)}&=&[
\bar{x}_1, \ldots, \bar{x}_{i-1}, \bar{x}_i+\Delta, \bar{x}_{i+1}, \ldots, \bar{x}_{j-1}, \bar{x}_j+\Delta, \bar{x}_{j+1}, \ldots,\bar{x}_d ]\\\nonumber
\bar{x}^{(i^{+},j^{-},\Delta)}&=&[
\bar{x}_1, \ldots, \bar{x}_{i-1}, \bar{x}_i+\Delta, \bar{x}_{i+1}, \ldots, \bar{x}_{j-1}, \bar{x}_j-\Delta, \bar{x}_{j+1}, \ldots,\bar{x}_d ]\\\nonumber
\bar{x}^{(i^{-},j^{+},\Delta)}&=&[
\bar{x}_1, \ldots, \bar{x}_{i-1}, \bar{x}_i-\Delta, \bar{x}_{i+1}, \ldots, \bar{x}_{j-1}, \bar{x}_j+\Delta, \bar{x}_{j+1}, \ldots,\bar{x}_d ]
\\\nonumber \bar{x}^{(i^{-},j^{-},\Delta)}&=&[
\bar{x}_1, \ldots, \bar{x}_{i-1}, \bar{x}_i-\Delta, \bar{x}_{i+1}, \ldots, \bar{x}_{j-1}, \bar{x}_j-\Delta, \bar{x}_{j+1}, \ldots,\bar{x}_d ]
\end{eqnarray}

Once $L^{\Delta}$ is obtained, we first run eigenvalue decomposition (EVD) of it, which is typically symmetric.
For symmetric matrices, singular value decomposition (SVD) and eigenvalue decomposition (EVD) are almost the identical, expect that the singular values are the absolute value of the eigenvalues, and therefore always non-negative.
The eigenvalue vector corresponding to the largest-magnitude eigenvalue value is the direction along which the function $y(x)$ changes most dramatically.
Therefore, the absolute value of this eigenvalue vector is chosen as the sensitivity measure $SI^{2,\Delta}$.
The probability of coordinate $i$ $(i=1, \ldots, n)$ selected to be perturbed can be   proportional to the magnitude of $SI_i^{2,\Delta}$. Here,  the above idea shares much in common with the EVD or SVD based principle component analysis, where the principle components (eigenvectors) explain most of the data variation. The difference is that we are not just generate the random candidate points exactly along the direction of the principle component. Instead, for each coordinate, the perturbation frequency when generating the random candidate is based on the principle component direction.

In this paper, $y(x)$ is chosen as the current response surface $s_n$, instead of the  original function $f(x)$, due to its cheap evaluation. Meanwhile a bigger step size $\Delta$ is allowed, since we are interested in the larger neighborhood of $\bar{x}$ and do not want to be stuck in the small vicinity of $\bar{x}$. 

Now we have two sensitivity measures $SI^{1,\Delta}$, and $SI^{2,\Delta}$.
 They might not coincide with each other, but they all  provide us with some useful ranking information of each coordinate.   Therefore, on each iteration $n\ge n_0$, some candidate points are generated following the information of $SI^{1, \Delta}$, some are following the information of $SI^{2,\Delta}$. 





%

\section{Computational Experiments} \label{Sec:Exp}
In this section, we will compare with state of the art alternative algorithms designed for minimizing computationally expensive functions, on extensive benchmark test problems and a real problems. A brief overview of these alternative algorithms is first given as follows: 
\subsection{Alternative Optimization Algorithms}
 The alternative algorithms to be compared include an original implementation of MSRS by \cite{Regis2007SRBF}, named LMSRBF,
 where all coordinates are perturbed simultaneously via normal distributions when generating a candidate point and the response surface is radial basis function interpolation.    We will briefly introduce other alterative algorithms, which come from many different categories of algorithms for global optimization.  Most of them have also been incorporating the idea of the response surface approximation  for the purpose of minimizing computationally expensive functions, except the dynamically dimensioned search algorithm (DDS, for short) developed by \cite{Tolson07DDS}, which is a response surface-free algorithm for computationally expensive objective functions. For the limitation of length, we could not cover all of the related algorithms. For example,   the RBF method by \cite{Gutmann01RBF} and its variants are not included in this comparisons partially because previous work by \cite{Regis2007SRBF} has showed that the LMSRBF algorithm performed much better than the RBF method by \cite{Gutmann01RBF} on a wide variety of test problems.




\subsubsection{Evolutionary Algorithms}
There are three main types of evolutionary algorithms, i.e.,
 genetic algorithms, evolution strategies and evolutionary programming algorithms. In this
 study, we choose 
   ESGRBF by \cite{Regis04IEEE} and SSKm by \cite{EGEA09SSKm}, since they have been tailored for minimizing computationally expensive functions by incorporating the response surfaces to mimic the original function $f(x)$.

ESGRBF belongs to the evolution strategies (ES, for short) and it establishes an RBF model using
information from previously evaluated points and uses it to screen out offspring
that are promising, thereby reducing the computational effort required in the ES.  For
a $(\mu, \lambda)$-Evolution Strategy (or simply $(\mu, \lambda)$-ES) (\cite{Back95evolutionstrategies}),  each generation consists of $\mu$ parents that produce $\lambda$ offspring via crossover and
 mutation. 
In the ESGRBF algorithm, we fit an RBF model  to screen out promising offspring.  In particular, an (m, $\mu$, $\lambda$, $\nu$)-ESGRBF is essentially a ($\mu, \lambda$)-ES that uses an RBF model to estimate the objective
function values of all the offspring, find out the $\nu$ offspring with the best RBF approximation values, and perform function evaluation of the  expensive objective function only on the selected offspring.  The $\mu$ parent solutions for the next generation
are then selected to be the best $\mu$ solutions from the $\nu$ offspring that are evaluated in the current
generation. The parameter $m$ is the size of the initial experimental design, which is  used to initialize the RBF model. 

Scatter search (\cite{Glover98SS})  is an another important evolutionary approach which originated from strategies for creating composite decision rules and response surface constraints. 
A new version of scatter search called SSKm (Scatter Search with Kriging for Matlab ) by \cite{EGEA09SSKm} presented for computationally expensive objective functions. The Kriging response surface implemented in SSKm helps avoid wasteful evaluations that are unlikely to provide high quality function values, thus effectively reducing the number of true function evaluations required to find the vicinity of the global solution.

\subsubsection{MultiStarts+LocalMinimization}


Some of the local minimization algorithms performed reasonably well and even
better than some algorithms meant for global optimization if the goal is to make quick progress
within a relatively limited number of function evaluations.  The multi-starting scheme is only performed to
 determine the starting points for the local minimization runs (i.e., through a space-filling design at
 the beginning or to select the points for the restarts), though the restarts is performed only very few times,
or even not any, due to the very limited number of function evaluations of $f(x)$.

 In the numerical experiments below, we use the following local minimization solver,  the
 NOMADm software  \url{http://www.gerad.ca/NOMAD/Abramson/nomadm.html},
 which is a Matlab implementation of the Mesh Adaptive Direct Search (MADS) algorithm by \cite{Audet06MADS}. In this study, we run NOMADm with the option  that uses the DACE response surface model by \cite{Lophaven02Kriging}  to make it suitable for computationally
 expensive functions. DACE stands for ``Design and Analysis of Computer Experiments," which is
 a methodology that uses kriging interpolation to model the outcome of a computer experiment.
 When searching over the Kriging response surface, we choose to use the Evolution Strategy with Covariance Matrix Adaptation (``CMA-ES''), which 
adopts a covariance matrix to explicitly rotate and scale the mutation distribution that works as strategy parameters.  This particular implementation of ``MADS'' was  referred as NOMADm-DACE in this paper.   The MADS algorithm has incorporated schemes to get out of the local trap to find an even better solution to a certain degree, though it is not necessarily  the global solution.



\subsubsection{EGO}
The efficient global optimization (EGO, for short) by  \cite{Jones98EGO} has been widely used for minimizing computationally expensive functions.
The EGO algorithm begins by first generating a number of
function evaluations at points generated via a Latin Hypercube (i.e., a space-filling design), and utilizes the DACE stochastic process model, i.e. Kriging to approximate $f(x)$ based on these available function evaluations, where ``DACE'' is an acronym for ``Design and Analysis of Computer Experiments''. EGO consists of two main components. The first one is that given data points, how to establish a Kriging model as follows:
\begin{equation} \label{eq:kriging}
r(x)= \mu +\sum_{i=1}^{p} b_i \exp \left[ -\sum_{h=1}^{d} \theta_h |x_h-x_h^{(i)}|^{p_h}\right]
\end{equation}
where $\theta_h \ge 0, p_h \in [0,2], h=1,\ldots,d$.
      The DACE (\ref{eq:kriging}) has up to $2d+2$ parameters: $\mu, \sigma^2, \theta_1, \ldots, \theta_d$ and $p_1, \ldots, p_d.$, and $b_i$ can be easily calculated once we have obtained these $2d+2$ parameters.  
%
The second component is related with the determination of the next function evaluation point based on the current Kriging model via
maximizing what is called ``the expected improvement'', which
aims to weigh
up both the predicted value of solutions, and the error in this
prediction, in order to automatically balance exploitation and exploration.  

\subsubsection{DDS}

Dynamically Dimensioned Search (DDS) algorithm, which is a response surface-free algorithm, was proposed by \cite{Tolson07DDS} and it main idea is to dynamically and randomly select the perturbed coordinates when generating the next function evaluation point based on the current best solution.
DDS
searches globally at the start of the search and becomes a
more local search as the number of iterations approaches the
maximum allowable number of function evaluations. The
transition from global to local search is achieved by
dynamically and probabilistically reducing the number of perturbed coordinates
from the current best solution.  The selection of the subset of coordinates for perturbation is completely at
random without reference to sensitivity information, and the perturbation is generated via normal distribution.
DDS is a greedy type of algorithm since the current
solution, also the best solution identified so far, is never
updated with a solution that has an inferior value of the
objective function.  Despite its simplicity,
 DDS has been shown
to be a very effective global optimization algorithm for computationally expensive objective functions in many applications.

\subsubsection{DYCORS}
DYCORS is a specific implementation of MSRS. The major difference
between DYCORS and LMSRBF is that DYCORS is more suitable for large dimensional problems ($> 30$ dimensions) since it does not
perturb all variables of the best point found so far in order to create candidate points, but rather borrow the idea of DDS, i.e.
each variable is perturbed with probability
$P(n)=p_0 \left[1-\frac{\log(n-n_0+1)}{\log(N_{\max}-n_0)}\right]$
for all $n_0 \le n  \le N_{\max}$, and where $n_0$ is the number of points in the initial experimental design, $p0 = \min(1, 20/d)$, $n$ is the iteration number, and $N_{\max}$ is the maximum number of allowed
evaluations for the optimization. Hence, the probability of perturbation for each variable decreases
as the optimization advances (as $n$ grows). It is ensured that at least one variable is perturbed.

\subsection{Test Problems}
\subsubsection{Synthetic  Problems}
These test problems are not really expensive to evaluate, but sill widely used for meaningful
 comparisons of performance for the different algorithms by pretending that these functions are
 computationally expensive. This can be done by keeping track of the best function values obtained
 by the different algorithms as the number of function evaluations increases. The relative performance
 of algorithms on these test problems are expected to be similar to the relative performance of these
 algorithms on truly expensive functions that have the same general shape as our test problems.
 Summary of test problems are in Table 1. While the main motivation of our research is on the relatively high dimensional problems, we still also test our new algorithm SO-SA on several low dimensional problems as well as high dimensional problems. We will see that SO-SA is also effective for low dimensional problems. 

\begin{table}
\centering
\begin{tabular}{|c|c|c|}
  \hline
  Test Funs & Domain & Global Min  \\\hline
   Ackely30 & $[-15,20]^{30}$ & -20-e \\
  Rastrigin30 & $[4,5]^{30}$ & -30 \\
  Levy30 & $[-5,+5]^{30}$ & $<-11$ \\
   Keane30 & $[1,10]^{30}$ & $<-0.39$ \\
  \hline
\end{tabular} \hspace{0.1cm}\begin{tabular}{|c|c|c|}
  \hline
  Test Funs & Domain & Global Min  \\\hline
   Michalewicz30 & $[0,\pi]^{30}$ & $<-23$ \\
  Schoen35 & $[0,1]^{30}$ & $<-80$ \\
  TB32& $[0,1]^{32}$ & 0\\
   \hline
\end{tabular}\label{tab:testproblems}\caption{Summary of test problems}
\end{table}



\subsubsection{Town Brook Watershed  Problem}
The Town Brook watershed is a 37 km2 subregion of the Cannonsville watershed (1200 km2)
in the Catskill Region of New York State. The time series Y of measured stream flows
and total dissolved phosphorus (TDP) concentrations used in the analysis contains 1096
daily observations (from October 1997 to September 2000) based on readings by the U.S.
Geological Survey for water entering the West Branch of the Delaware River from the Town
Brook watershed. We used the SWAT2005 simulator (\cite{ARNOLD12SWAT}), which has been
used by over a thousand agencies and academic institutions worldwide for the analysis of water flow and nutrient transport in watersheds. 
The water draining the
Town Brook and rest of the Cannonsville watershed collects in the Cannonsville Reservoir,
from which it is piped hundreds of miles to New York City for drinking water. Water quality
is threatened by phosphorus pollution and, if not protected, could result in the need for a
New York City water filtration plant estimated to cost over $8$ billion. For this economic
reason as well as for general environmental concerns, there is great interest in quantifying
the parameter uncertainty for this model. The input information of the Town Brook simulator is discussed briefly in \cite{Tolson07DDS} and in more details in \cite{Tolson2007Swat}. From the computational point view, in total $32$ parameters normalized to [0, 1] will be calibrated and the objective function is the sum of $4$ terms which measure the relative distance between the model outputs and the measure data, corresponding to flow, dissolved phosphorus, sediment, and organic nitrogen transported with water, respectively.

\subsection{Experimental Setup}
\subsubsection{Parameter Settings of Involved Algorithms }
We perform all numerical computations in Matlab (R2011a) on a Dell workstation with  2.66Ghz CPU and the 32 GB of RAM. The operating system is 64-bit Windows sever 2008 R2 standard. We compared SO-SA with $6$ alternatives: LMSRBF, DDS, EGO, SSKm, NOMADm-DACE, ESGRBF.

For LMSRBF, EGO, and ESGRBF, SO-SA, a Latin hypercube design is used to initialize the RBF model with size $m=2(d+1)$.
For LMSRBF, ESGRBF, SO-SA, we all use a cubic RBF model to approximate the expensive objective function in every iteration.

As for the the ESGRBF algorithm, we set the number of parents $ \mu= 8$; the number of offspring $\lambda= 50$ and $\nu= 20$ offspring is chosen according to the RBF model.

SSKm contains several steps. In the improvement step, a local search is implemented using a carefully selected starting point and there are many options for the local search methods such as fmincon in the Matlab optimization toolbox, fminsearch, NOMADm, solnp (a SQP method by  \cite{Ye97SQP}). In this paper, we are using fmincon.  In addition,  it contains a  kriging interpolation.  Its parameter parameter is by default  performed by fminsearch, whose the maximum number of function evaluations is by default 200*d. For the rest settings, we also follow its default settings. 

As for EGO,  we adopted the implementation in the SURROGATES Toolbox developed by \cite{manual:surrogatestoolbox2p1}. It used the DACE toolbox \url{http://www2.imm.dtu.dk/~hbn/dace/} to establish the Kriging model in each iteration.  When maximizing the ``the expected improvement''  to determine the next function evaluation point, it used a differential evolution algorithm by \cite{Price05DE}. The default settings are used in the following experiments. 

As for ``NOMADm-DACE", when  choosing the search strategy, we select the option of makes use of surrogates, since we  are considering the computationally expensive functions. Among the available
 surrogates, we adopt the DACE, which aims to establish a Kriging surrogate. When searching over the surrogate, we choose to use
 ``CMA-ES", since it supports global search.  The corresponding parameters are kept as defaults.

Since we are considering the computationally expensive functions, we are mainly caring about the number of true function evaluations. In fact,
besides these expensive function evaluations, the overhead computation of SO-SA is still much less than EGO, KKim, NOMADm-DACE-CMAES and ESGRBF, though slightly higher than DDS, LMSRBF, as showed in the section \ref{sec:overhead}.

\subsubsection{Evaluation Criteria}

There are never likely to be fully accepted automated stopping criteria for stochastic search algorithms.  For that reason, we will generally emphasize algorithm comparisons and stopping based on ``budgets" of function valuations in this paper. Given the number of function evaluations which is usually not big due to the high computationally cost of the function evaluation, we compare the lowest objective function values that different algorithms can achieve.

\subsection{Results on Test Problems}

The results of the proposed SO-SA was compared with these alternatives are plotted in Figures \ref{fig:Ackley30}, \ref{fig:Ras30}, \ref{fig:Mich30}, \ref{fig:GWSS35}, \ref{fig:Keane30}, \ref{fig:Levy30},  and \ref{fig:TB32}. For each figure, the horizontal axis represents the number of performed function evaluations of $f(x)$ and the vertical axis is the current minimal function  value $f_n^*$ among already performed function evaluations. For each problem, $30$ runs are performed for each algorithm since they are mostly stochastic algorithms. The averaged current lowest function values and the corresponding standard deviations are plotted.  We can see that for all these test problems, our algorithm SO-SA always achieves lowest objective function value, and achieve significant improvements over alternatives mostly. The corresponding size of standard deviation is comparable with these alternatives, if not smaller.
\begin{figure}[h!]
\centering
\includegraphics[scale=1]{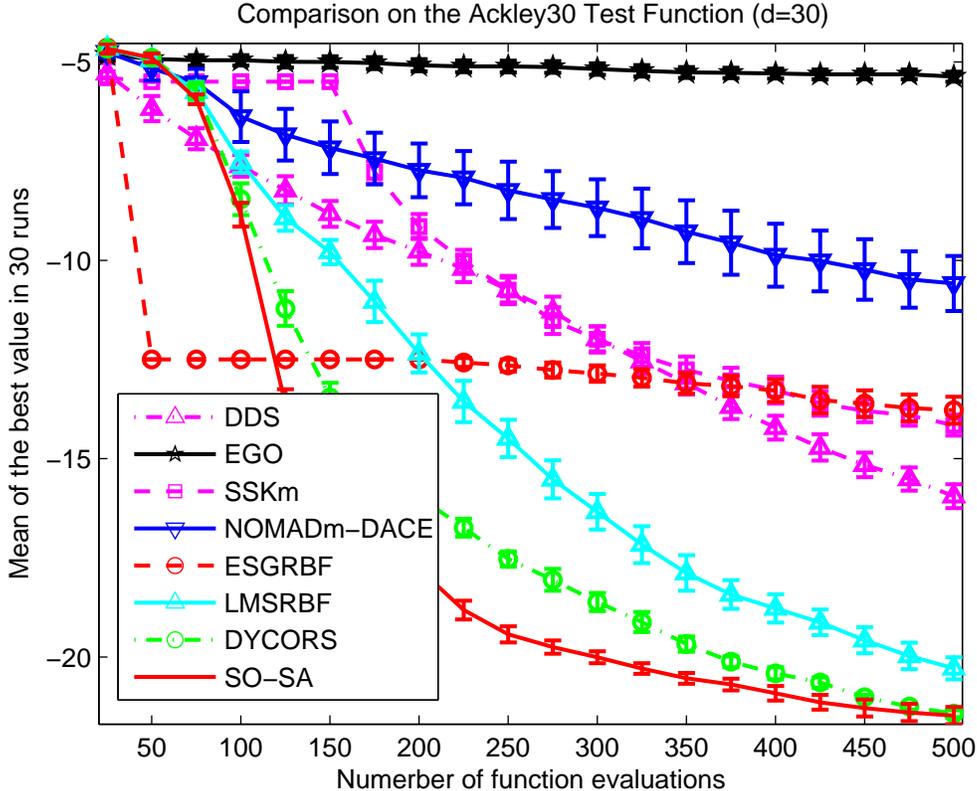}\hspace{0cm}
\caption{Comparison of Global Optimization Methods on the
Ackley Function ($d =30$)} \label{fig:Ackley30}
\end{figure}

\begin{figure}[h!]
\centering
\includegraphics[scale=1]{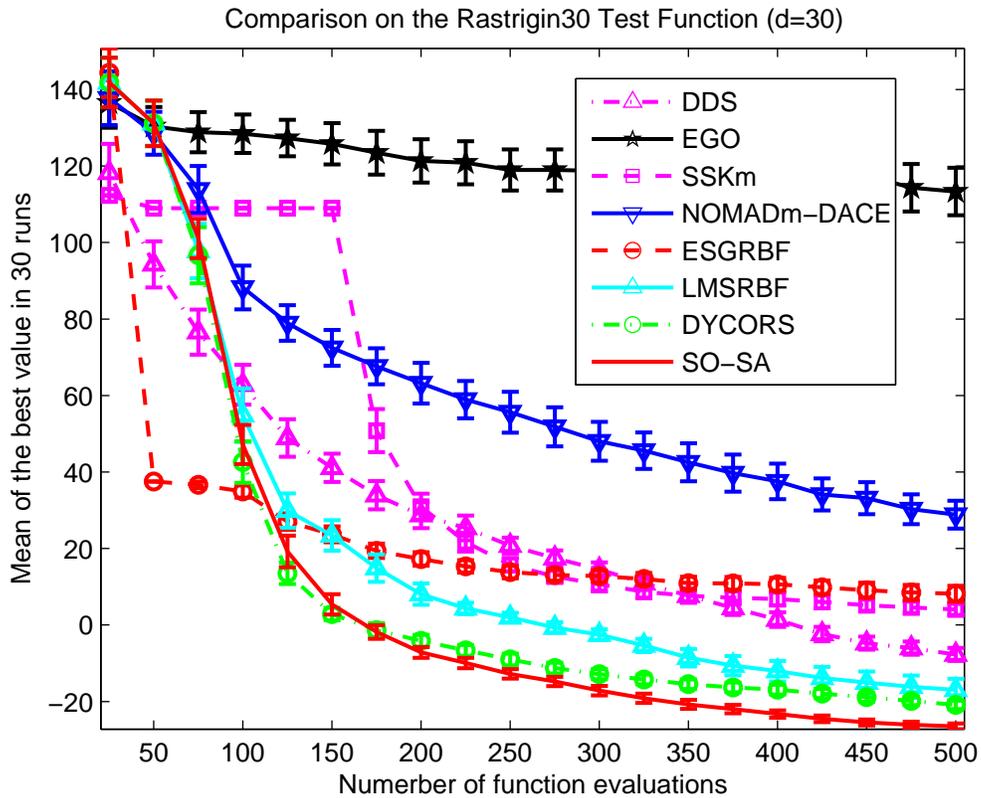}\hspace{0cm}
\includegraphics[scale=1]{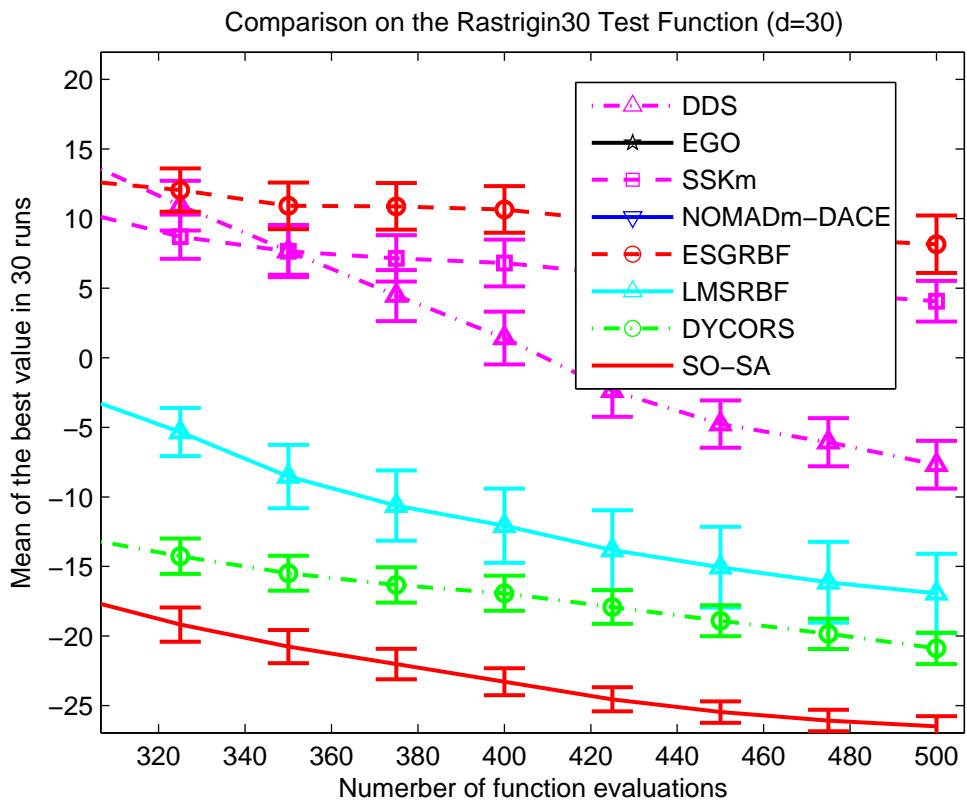}\hspace{0cm}
\caption{Comparison of Global Optimization Methods on the
Rastrigin Function ($d =30$)} \label{fig:Ras30}
\end{figure}

\begin{figure}[h!]
\centering
\includegraphics[scale=1]{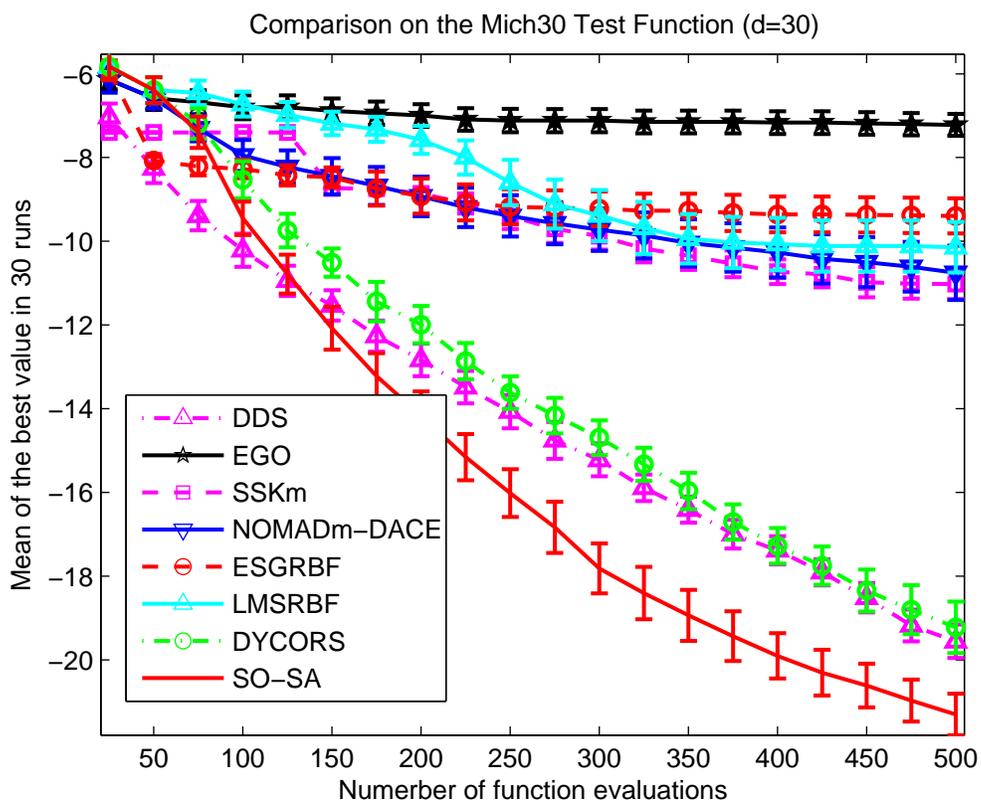}\hspace{0cm}
\caption{Comparison of Global Optimization Methods on the
Michalewicz Function ($d =30$)} \label{fig:Mich30}
\end{figure}

\begin{figure}[h!]
\centering
\includegraphics[scale=1]{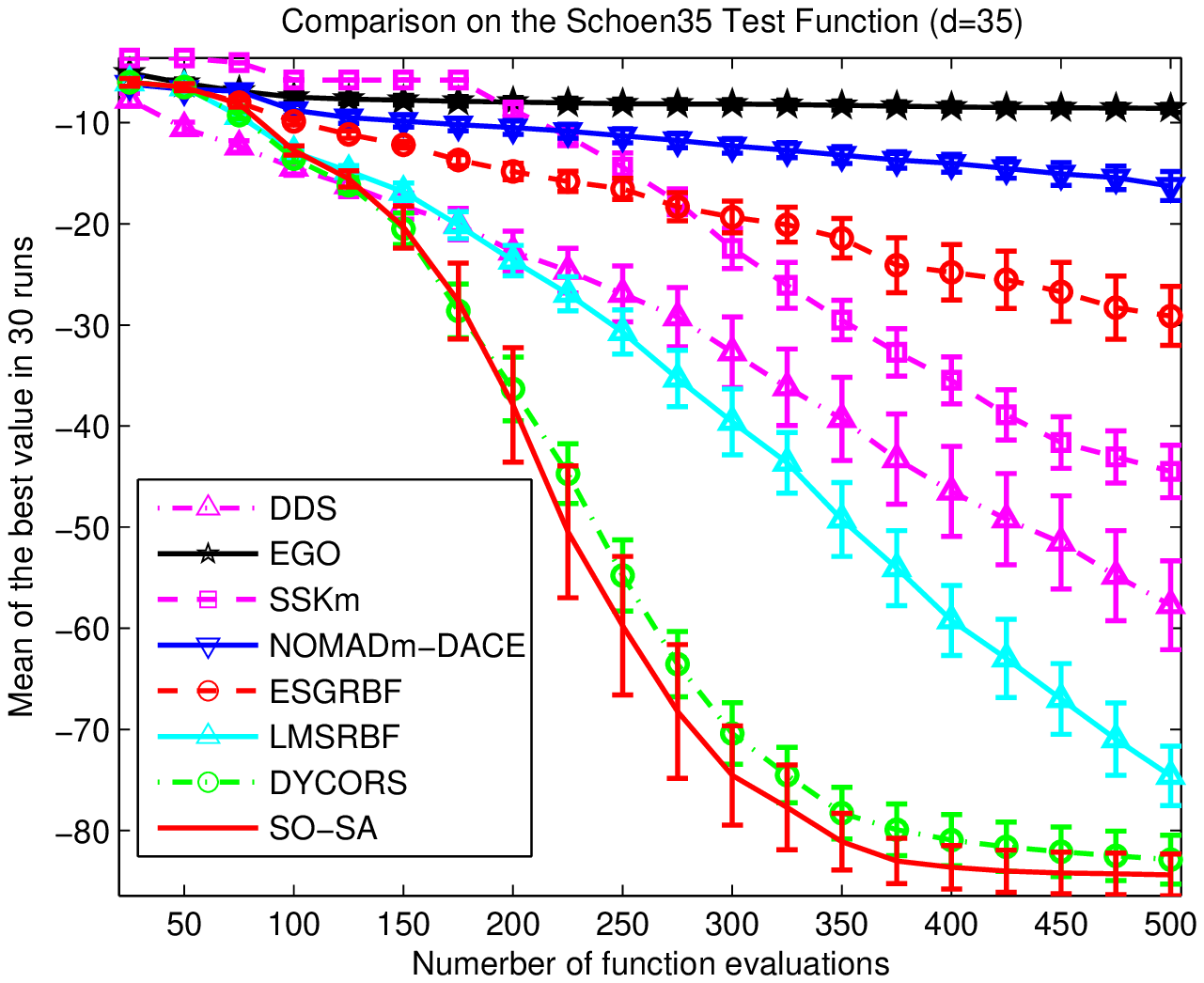}\hspace{0cm}
\includegraphics[scale=1]{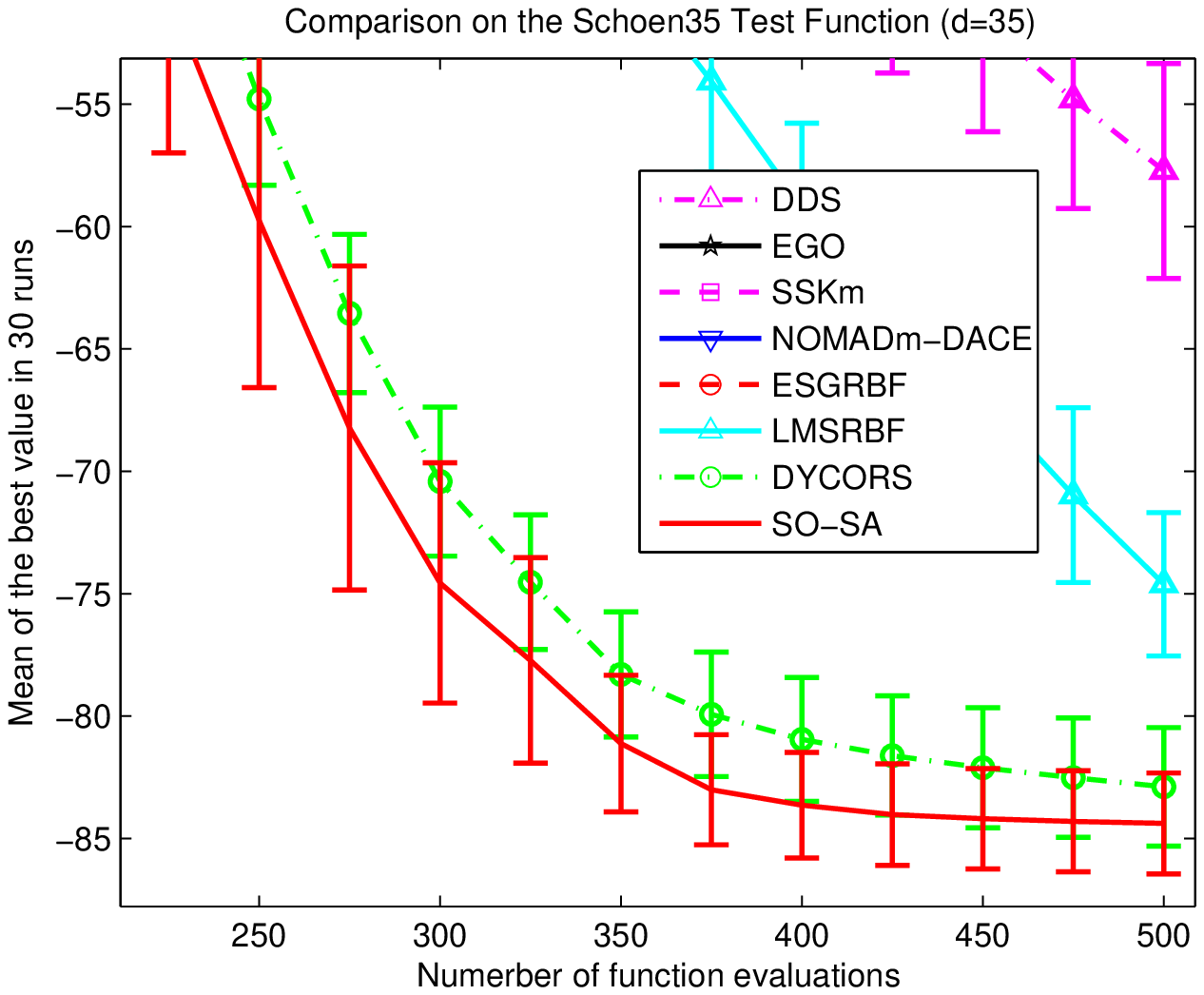}\hspace{0cm}
\caption{Comparison of Global Optimization Methods on the
Schoen Function ($d =35$). The lower one is the zoom in version.} \label{fig:GWSS35}
\end{figure}

\begin{figure}[h!]
\centering
\includegraphics[scale=1]{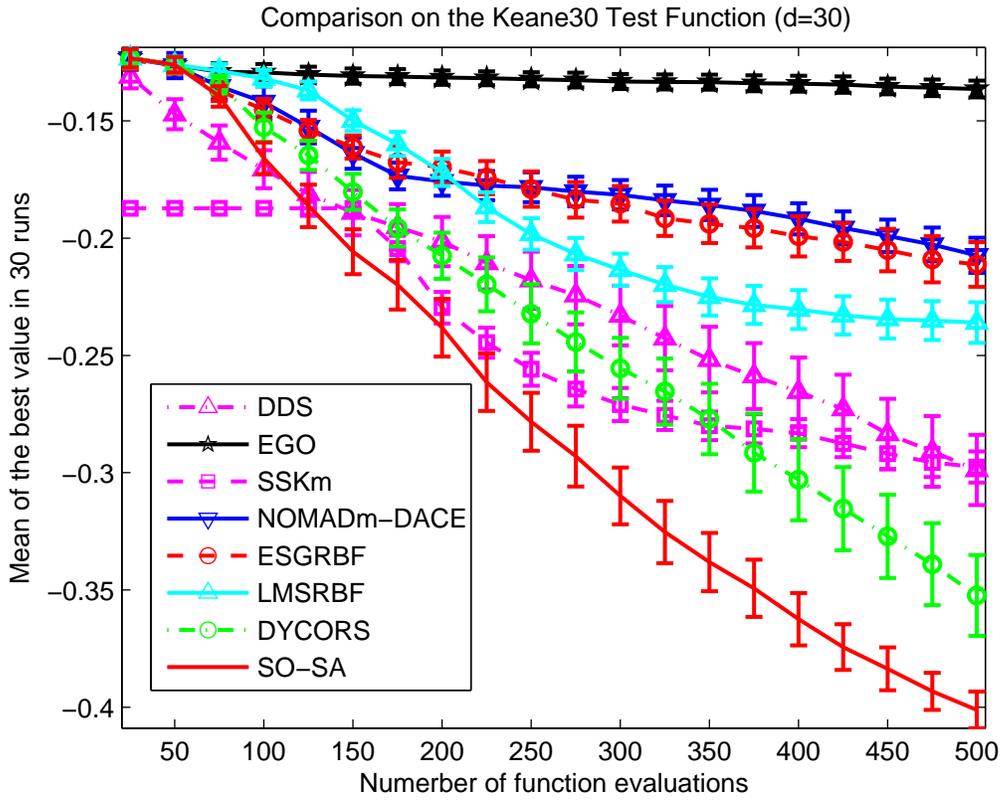}\hspace{0cm}
\caption{Comparison of Global Optimization Methods on the
Keane Function ($d =30$)} \label{fig:Keane30}
\end{figure}

\begin{figure}[h!]
\centering
\includegraphics[scale=1]{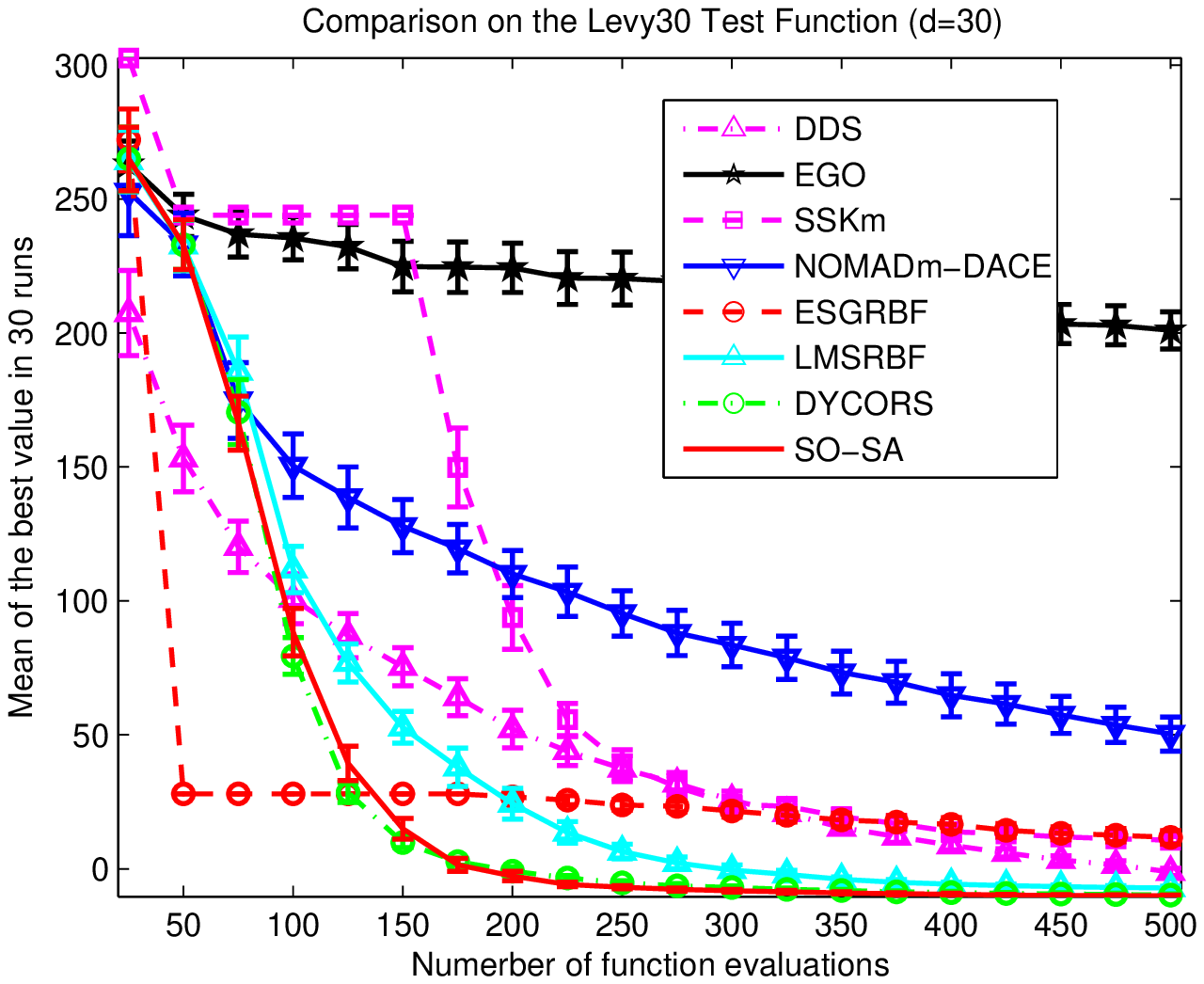}\hspace{0cm}
\includegraphics[scale=1]{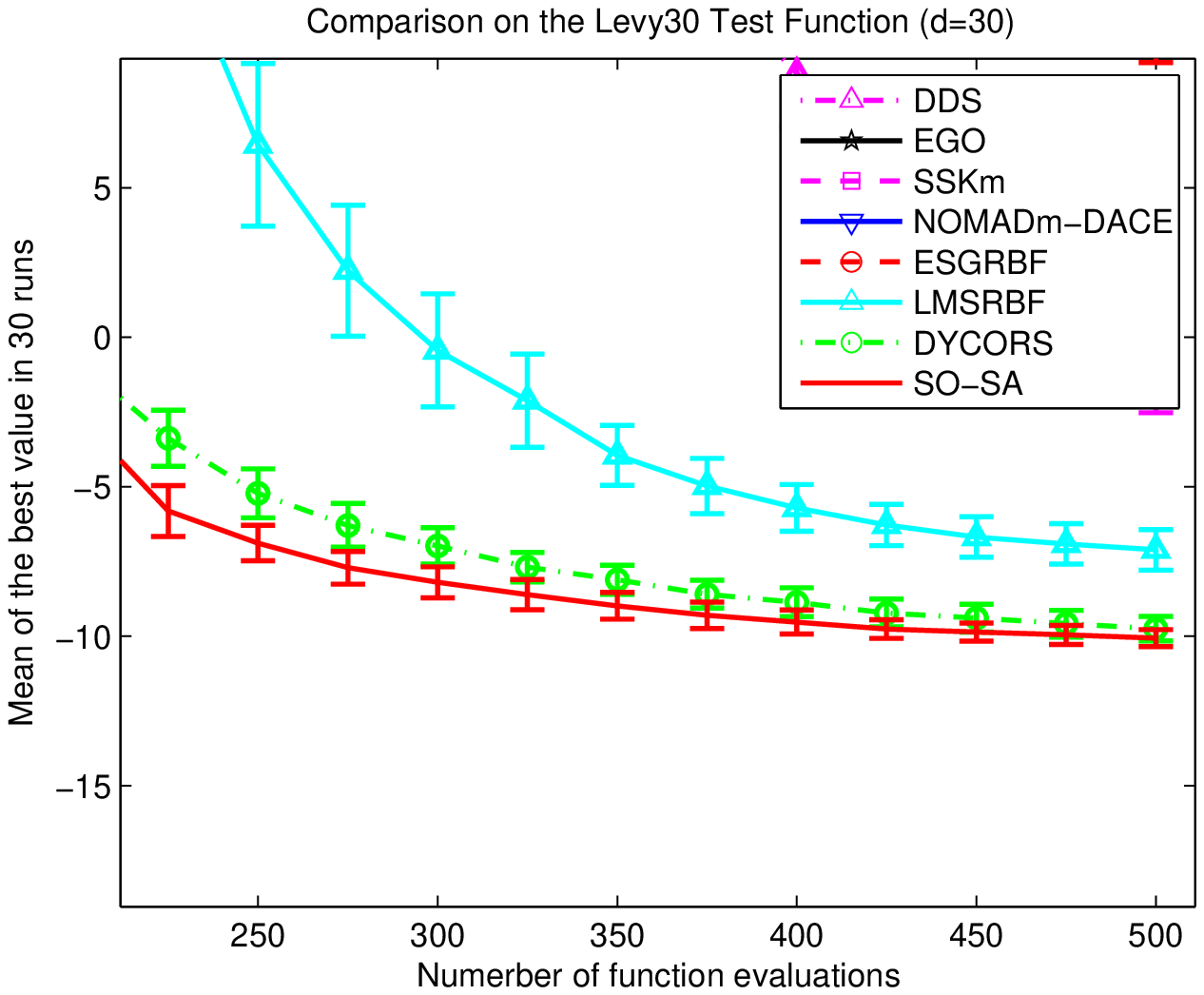}\hspace{0cm}
\caption{Comparison of Global Optimization Methods on the
Levy Function ($d =30$). The lower one is the zoom in version.} \label{fig:Levy30}
\end{figure}

\begin{figure}[h!]
\centering
\includegraphics[scale=1]{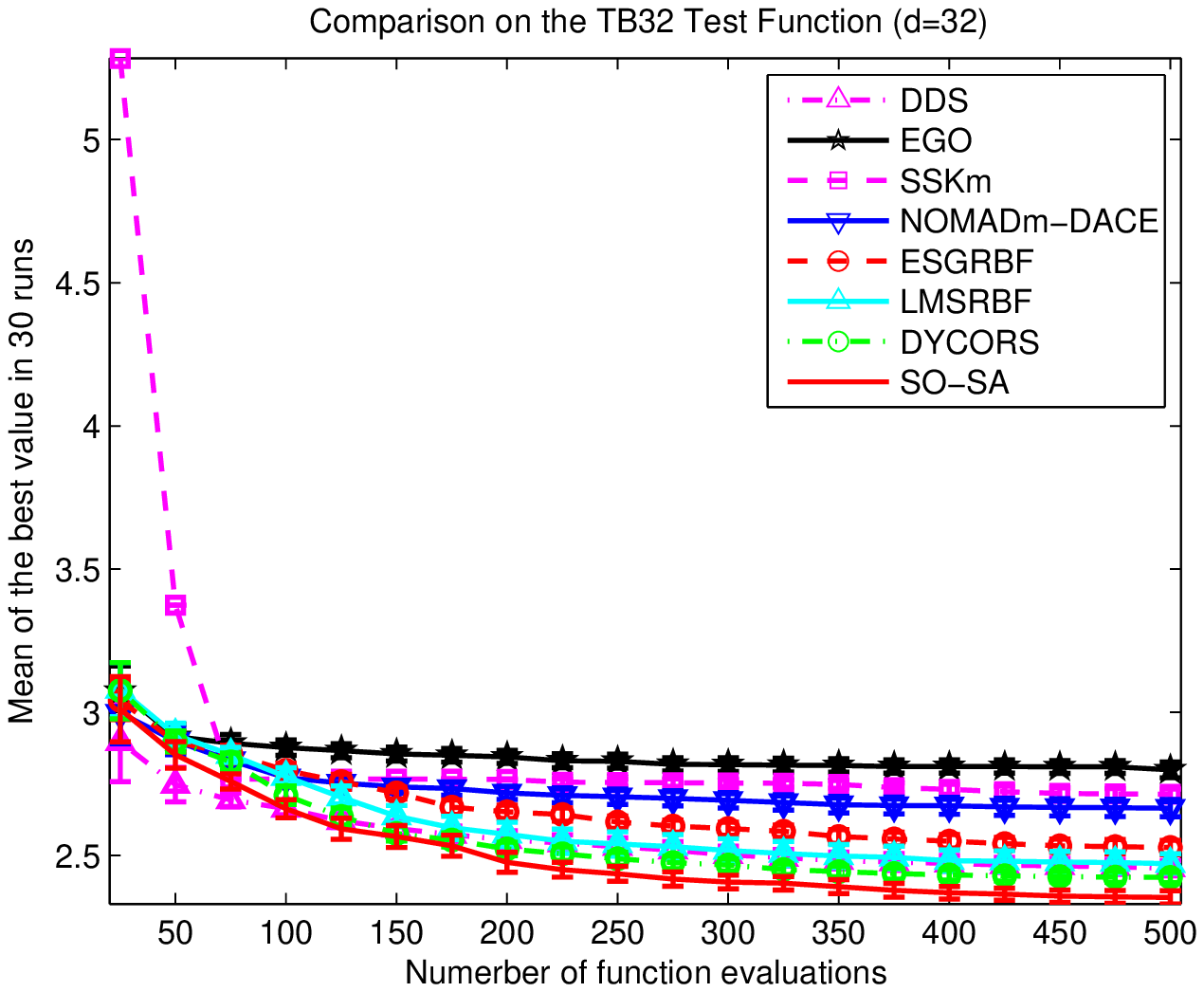}\hspace{0cm}
\includegraphics[scale=1]{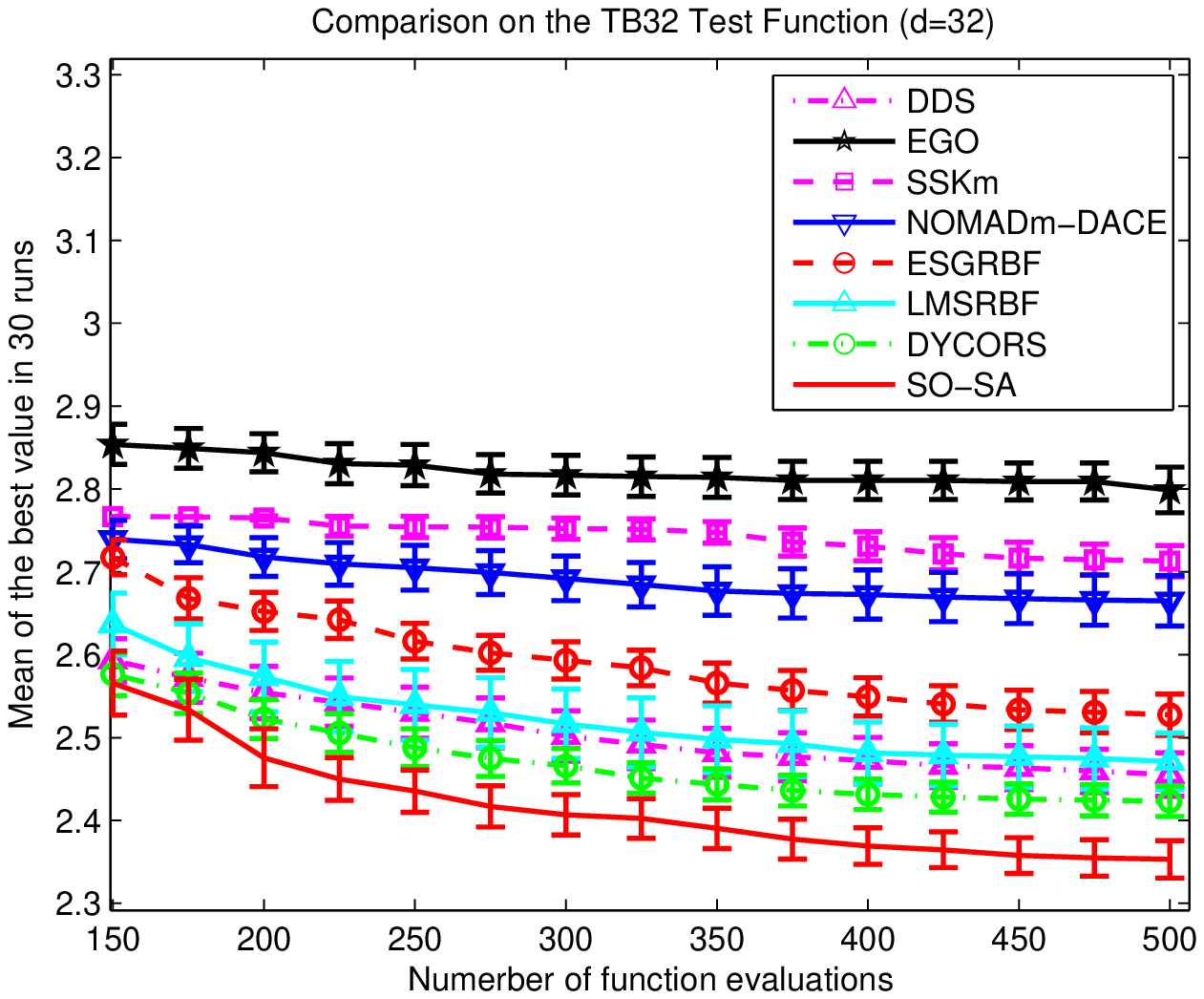}\hspace{0cm}
\caption{Comparison of Global Optimization Methods on the
Town Brook Problem ($d =32$). The lower one is the zoom in version.} \label{fig:TB32}
\end{figure}



We further  proposed a quantify $Q(A, P)$ to better describe the relative performances of different algorithms, where
$Q(A, P)$=the relative difference between the objective function value by algorithm $A$ on
problem $P$ and the best objective function value among all algorithms on problem $P$. That is to say, $Q(A,P)=\frac{|f_{best}^{A,P}-f_{best}^{P}|}{|f_{best}^{P}|}$ where $f_{best}^{P}=\min_{A} f_{best}^{A, P}$.
Let $Q(A)= \sum_{P} Q(A, P)$ and the smaller $Q(A)$ indicates the better performance of the algorithm $A$. Table \ref{tab:summary} summarized the $Q$ values for different algorithms given these testing problems. We can see that SO-SA performs the best compared with   these alternatives. We would like to point out EGO performs worse in our cases of high dimensional problems.

\begin{table}\caption{Summary of Running Results}
\begin{tabular}{|l|l|l|l|l|l|l|l|l|}
  \hline
 &  SO-SA &DYCORS&	LMSRBF	& DDS	&SSKm	&ESGRBF	&NOMADm	&EGO\\\hline
Ackley30&  -21.55 &	-21.41 	& -20.29	& -15.94	&  -14.15	 &-13.78&   -10.68	& -5.37\\\hline
Rastrigin30&-26.48& -20.88	&-16.94	& -7.69	& 4.07	& 8.16	& 28.83& 113.32\\\hline
Michalewicz30&-21.30& -19.23&	  -10.14	&-19.56& -11.02	& -9.40	&-7.22&	-7.39\\\hline
Keane30&- 0.40&	-0.35 &    -0.24	&-0.30	&-0.30	&-0.21	&-0.21&	-0.14\\\hline
Levy30 &-10.45&	 -9.75&     -7.11	&-1.14	&10.76	&11.75	&50.32&	200.01\\\hline
Schoen35&-84.39& -82.89	& -74.61	& -57.72& -44.49	&  -29.11	& -16.26&	 -8.58\\\hline
TB32&2.35 &2.43	 &2.47& 2.46	& 2.71	& 2.53&2.66 &2.79	  \\\hline
 Q(A)& 0  &  0.56  & 1.83           & 2.56      & 4.88      & 5.56      &   10.48  &  28.56  \\\hline
\end{tabular}\label{tab:summary}
\end{table}

\subsection{Short Note about the Overhead Running Time} \label{sec:overhead}
When minimizing the  computationally expensive objective function, we usually assume that the running time is mostly spent on the evaluations of the objective function $f(x)$.
Specifically, the overhead time including establishing the response surface and determining the next function evaluation is considered to be ignorable compared with the evaluations of $f(x)$. That is to say, this kind of overhead time is mainly for the algorithm to ``think" where to perform the evaluations of $f(x)$.   However, this is not always true, and  the ``thinking" time for different algorithms are quite different and even big and Table \ref{tab:overhead} is the overhead CPU time of our tested algorithms for one run where  $500$ function evaluations are performed. The ``thinking" time of our algorithm is moderate although longer than LMSRBF and DYCORS as expected due to the more complicated schemes for determine the next function evaluation point. Notice that while the ``thinking" time of SSKm and EGO is significantly longer than others, their results are  among the worst from Table \ref{tab:summary}, especially for EGO. The overhead time is independent of the running time of each evaluation  $f(x)$ and therefore whether it is ignorable often relies on the evaluation time of $f(x)$.

{\tiny
\begin{table}\caption{Summary of Overhead Time (s)}
\begin{tabular}{|l|l|l|l|l|l|l|l|l|}
  \hline
  SO-SA&DYCORS&	LMSRBF	& DDS	&SSKm	&ESGRBF	&NOMADm	&EGO\\\hline
   $1.2\times 10^3$  &350  &   300        & 0.18    &   $7.1\times 10^5$   &     2.0 & 67    & $1.8\times 10^4$   \\\hline
\end{tabular}\label{tab:overhead}
\end{table}}

\section{Conclusions}\label{Sec:future}
Response surface based global optimization algorithms have been playing a very important role for computationally expensive objective functions, arising from many practical problems, for example, parameter calibration of complex physical models. In this paper,  a stochastic response surface method named ``MSRS" proposed by \cite{Regis2007SRBF} where the next function evaluation point is chosen from a set of random candidate points, was further analyzed and extended.  Specifically,  we propose a new way to balance the local searching and global exploration by extending the way of generalizing the random candidate points via setting different probability values for each coordinate to be perturbed.  The smaller probability value is more likely to prefer local searching and vice versa. Correspondingly we present a new definition of ``neighborhood" by the number of perturbed coordinates. Based on the above new ideas,
we finally proposed a specific implementation of ``MSRS", which takes sensitivity information
into consideration when selecting the perturbed coordinate for producing the random candidate points.  Its outstanding performance over many state of the art algorithms is well demonstrated by many typical testing problems, especially for high dimensional problems. 
In the future, we will further study the features of the method of the random candidate points and extend it to many other existing algorithms, for example, EGO, which does not work well for high dimensional problems in our experiments.

\section{Acknowledgement} 
This work was supported by the Natural Science Foundation of China, Grant
Nos. 11201054, 91330201 and by the Fundamental Research Funds for the Central Universities
ZYGX2012J118, ZYGX2013Z005.


%
%
%




\bibliographystyle{ijocv081}
\bibliography{references_Surrogate_OPT,references_GO,references_kriging,RandomSearch}

\end{document}